\documentclass[11pt,fleqn]{article}
 
 \usepackage{fixltx2e} 
\usepackage{booktabs} 
\usepackage{array} 
\usepackage{caption}
\usepackage{float}
\usepackage{paralist} 
\usepackage{verbatim} 
\usepackage{subfigure} 
\usepackage{rotating}
\usepackage{color}
\usepackage{epic,eepic}
\usepackage{geometry} 
 \usepackage{fancyhdr} 
 \usepackage{tikz}
\usetikzlibrary{shapes,backgrounds,patterns,arrows}
\usepackage{linguex}
\alignSubExtrue
\usepackage{url}
\usepackage{relsize}

\usepackage{latexsym,amssymb,amsmath,amsthm}
\usepackage{txfonts}
\usepackage{textcomp} 
\usepackage{epstopdf} 
\usepackage{flafter}  

\usepackage{bm}  

\usepackage{natbib} 

\usepackage{amsmath,amssymb,latexsym}  
\usepackage{bm}  

\newtheorem{definition}{Definition}
\newtheorem{theorem}{Theorem}
\newtheorem{proposition}{Proposition}

\newtheorem{lemma}{Lemma}

\usepackage[stable]{footmisc}
\usepackage{memhfixc}  

\usepackage{bussproofs} 
\usepackage{fitch} 

\usepackage{stmaryrd}  
\newcommand{\sem}[1]{\llbracket #1 \rrbracket}

\newcommand{\T}{{\bf T}}
\newcommand{\F}{{\bf F}}
\newcommand{\Ta}{{\bf T} }
\newcommand{\Fa}{{\bf F} }

\newcommand{\N}{{\bf N}}
\newcommand{\Na}{{\bf N} }

\newcommand{\B}{{\bf B}}
\newcommand{\Ba}{{\bf B} }

\newcommand{\cmi}{\rightarrow_{cmi}}

\title{K3, \L3, LP, RM3, A3, FDE, M: \\How to Make Many-Valued Logics Work for You}
\author{Allen P. Hazen\\Dept. Philosophy\\University of Alberta \and Francis Jeffry Pelletier\\Dept. Philosophy\\University of Alberta}

\date{Version of September 18, 2017}
\begin{document}
\maketitle

\begin{abstract}
\noindent We investigate some well-known (and a few not-so-well-known) many-valued logics that have a small number (3 or 4) of truth values.  For some of them we complain that they do not have any \emph{logical} use (despite their perhaps having some intuitive semantic interest) and we look at ways to add features so as to make them useful, while retaining their intuitive appeal.  At the end, we show some surprising results in the system FDE, and its relationships with features of other logics.  We close with some  new examples of ``synonymous logics.''  An Appendix contains a natural deduction system for our augmented FDE, and proofs of soundness and completeness.
\end{abstract}

\section{Truth-Values: Many, versus Gaps and Gluts}   
Everyone knows and loves the two ``classical'' truth-values, {\bf T}rue and {\bf F}alse.  In ``classical'' logic\footnote{There are those (e.g., \citealp{Pri06}) who claim that this common terminology of `classical' betrays a bias and is in fact incorrect as an account of the history of thought about logic.  We won't attempt to deal with this, and will just use the common terminology.  For a more comprehensive review of topics and issues relevant to truth values see \citet{SW17truth-values}}, every sentence is either {\bf T} or {\bf F} but not both.  Before getting to the detailed issues of this paper, we pause to make four comments of a ``philosophical nature'' with the intent of setting the issues aside.
First, there is the long-standing issue of whether a logic contains sentences, statements, or propositions.  We take no position on this issue: what we think of as the \emph{logical} point is neutral on this issue.  
Related to this, perhaps, is the issue of whether a logic should be considered a set of its theorems or defined instead by the valid inferences it accepts.  This topic \emph{does} mark a distinction in the three-valued cases we consider, and we will remark on this as the occasion arises.  Second, there are logics -- perhaps not ``classical'' ones -- that would prefer to claim the objects of interest are (e.g.)~obligations, or commands, or ethical/epistemic items, rather than {\bf T} and {\bf F}.  Perhaps the items that have such values are not to be considered `sentences' but rather exhortations or promises or\ldots?  If so, then we say that we are restricting our attention to (declarative) sentences. Third, while `sentence' seems appropriate for propositional logic, some further accommodation is required for formulas of predicate logic.  The present paper is mostly concerned with propositional logic, and so we will simply use `sentence'.  And finally, there is the philosophical issue of whether {\bf T} and {\bf F} should be considered to be \emph{objects}-- perhaps existing in a ``third realm'' -- or as \emph{properties} of sentences.  Our view is that this distinction -- which is maybe of considerable philosophical interest -- does not affect any of the claims we will be making, and so we just will pass over it in silence.

So, given that \Ta and \Fa are so revered, where does this leave many-valued logic?  One intuition that has been held by very many theorists is that there simply are more than just the \Ta and \Fa  truth values.  Some of these theorists hold that various phenomena show that there is another category beyond \Ta and \F:  perhaps we should call it {\bf I} for {\bf I}ndefinite.  Presupposition failures in a formula might determine that it is {\bf I}; vagueness might determine that a formula is {\bf I}; the unactualized (as yet) future might determine that some formulas are {\bf I}; fictional discourse seems neither straightforwardly true or false, so perhaps sentences about fictional objects also should be {\bf I}.  And there no doubt are other types of formulas that could intuitively be thought to be ``indeterminate''.
\ex. \label{intermediate}\a. The present Queen of Germany is happy.
\b. That 45-year-old man is old.
\c. There will be a Presidential impeachment in the next 30 years.
\b. James T. Kirk is captain of the starship \emph{Enterprise}.

Some theorists, agreeing with the just-expressed intuition that the sentences in \ref{intermediate} are neither \Ta nor \F, nonetheless wish to deny that they designate a third value.  Rather, they say, such sentences \emph{lack} a truth-value (i.e., are \emph{neither} \Ta nor \F; nor are they {\bf I}, since they deny such a value).  Such sentences \emph{have no truth-value}:  they express a truth-value \emph{gap}.  And logics that have this understanding are usually called ``gap logics.''

Formally speaking, there is little difference between the two attitudes of expressing {\bf I} or being a gap.  In both cases one wishes to know how to evaluate complex sentences that contain one of these types of sentences as a part.  That is, we need a way to express the truth-table-like properties of this third option.  In the remainder of this paper we will use \Na for this attitude, whether it is viewed as expressing {\bf I} or as being a gap.  (The \Na can be viewed as expressing the notion ``neither \Ta nor \F'', which is neutral between expressing {\bf I} or expressing a gap.)  And we will call the logics generated with either one of these understandings of \N, ``gap logics.''

Gap logics -- whether taken to allow for truth values other than \Ta and \F, or instead to allow for some sentences not to have a truth value -- are only one way to look at some recalcitrant natural language cases.  Other theorists point to some different phenomena in their justification of a third truth value.  The semantic paradoxes, such as the Liar Paradox (where
the sentence ``This sentence is false'' is false if it is true, and is true
if it is false) suggest that some statements should be treated as \emph{both} true and false.  It is difficult to deal with this paradox in a gap logic, because saying that ``This sentence is false'' is neither true nor false leads immediately to the strengthened Liar Paradox:  ``Either this sentence is  false or else it is neither true nor false''.  Here it seems that claiming that the sentence expresses a gap leads to truth, while saying that it is true leads to falsity and saying that it is false leads to truth.  In contrast, to say that the sentence ``This sentence is false'' is \emph{both} true and false does not lead to this regress:  The sentence ``This sentence is either false or else is both true and false'' leads only to the conclusion that it is both true and false.  This \emph{dialetheic} theory -- that some sentences are both true and false -- is called a ``glut theory'' because of the presence of both truth values characterizing some sentences.  The intermediate value, on this conception, can be called \Ba (for ``both''), and we will use this.

There are also examples from na\"{i}ve set theory such as
the Paradox of Well-Founded sets, 
the Paradox of a Universal Set, 
Russell's Paradox, 
and
Richard's Paradox 
(see \cite{Can14} for a historical overview of these).  There is a desire on the part of some theorists in the philosophy of mathematics  to reinstate na\"{i}ve set theory and use it in the development of mathematics, instead of (for example) Zermelo-Frankel set theory.  Such theorists perhaps will feel encouraged in their belief that the glut-style of resolution of the Liar Paradox might be usable in the case of these other set theoretic paradoxes, and that na\"{i}ve set theory can once again be used.

Various theorists have pointed out that vagueness needn't be viewed as \emph{missing} a truth-value; it is equally plausible to think of vagueness as manifesting \emph{both} the positive character and its absence.  A middle-height person can be seen as \emph{both} short and tall, equally as plausibly as being \emph{neither} short nor tall.  (See \cite{Hyde97,BC01,HC08}.)  As well, certain psychological studies seem to indicate that the ``dialetheic answer'' is more common than the ``gap'' answer, at least in a wide variety of cases (see \cite{AP11,Ripley11} for studies that examine people's answers to such cases).

There are also other categories of examples that have been discussed
in philosophy over the ages:
Can God make a stone so heavy that He can't lift it?
Perhaps the answer is, ``Well, yes and no.''
There are cases of vague predicates such as \emph{green} and \emph{religion}, 
e.g., an object can be both green and not green, and a belief system can 
be both a religion and not a religion.
There are legal systems that are inconsistent, in which an action is both
legal and illegal.
In the physical world, there is a point when a person is walking through a
doorway at which the person is both in and not in the room.
And so on.
As a result, the glut theory has its share of advocates.  

A way to bring the gap and glut views of the truth values under the same conceptual roof is to think of the method that assigns a truth value to formulas as a \emph{relation} rather than a \emph{function}.  A function, by definition, assigns exactly one value to a formula -- either \Ta or \Fa or, in the context of our 3-valued logics, it may also assign \Ba or {\bf I}.  A relation, however, can assign more than one value: we may now think that there are exactly two truth values, \Ta and \F, but the assignment \emph{relation} can assign subsets of $\{\T,\F\}$.  A gap logic allows assignment of $\{\T\}$, $\{\F\}$, and $\emptyset$.  A glut logic allows $\{\T\}$, $\{\F\}$, and $\{\T,\F\}$.

But both the gap and the glut logics seem to have difficulties of a logical nature that we will examine in \S\ref{Wrong}.

\section{K3 and LP: The Basic 3-Valued Logics}
In all of the 3-valued logics we consider, the 3-valued matrices for $\land,\lor,\neg$ are the same, except for what the ``third value'' is called: in K3 \citep[\S64]{Kleene1952} 
we call it \Na but  in LP \citep{Asenjo66,Priest79} we call it \B.  

One way to understand why the K3 truth values for these connectives are correctly given in the following tables is to think of \Na as meaning ``has one of the values \Ta or \F, but I don't know which one''.  In this understanding, a negation of \Na would also have to have the value \N; a disjunction with one disjunct valued \Na and the other \Ta would  as a whole be \T \ldots but if one were valued \Na and the other \F, then as a whole it would have to be \N.  Similar considerations will show that the $\land$ truth table is also in accord with this understanding of the \Na value.  (This interpretation of the truth values presumes that one pays no attention to the \emph{interaction} of the connectives:  $(p\land\neg p)$ seems to be false and $(p\lor\neg p)$ seems true, but that's because of the interaction of two connectives.)

On the other hand, if the middle value is understood as being \emph{both} true and false as in LP, we will still get those same truth tables, with only a change of letter from \Na to \B:  negating something that is both true and false will yield something that is both false and true.  If a disjunction had one disjunct valued \Ba and the other \T, then the entire disjunction would be valued \T.  But if that second disjunct were valued \Fa instead, then the entire disjunction would be valued \B.  Similar considerations will show that the $\land$ truth table is also in accord with this understanding of the \Ba value.

\bigskip

\begin{tabular}{c| c c c}
$\land$ & \T & \N/\B & \F \\
\hline
\T & \T & \N/\B & \F \\
\N/\B & \N/\B & \N/\B & \F \\
\F & \F & \F & \F \\
\end{tabular}
\qquad
\begin{tabular}{c| ccc}
$\lor$ & \T & \N/\B & \F \\
\hline
\T & \T & \T & \T \\
\N/\B & \T & \N/\B & \N/\B \\
\F & \T & \N/\B & \F \\
\end{tabular}
\qquad
\begin{tabular}{c| c}
$\neg$ & \\
\hline
\T & \F \\
\N/\B & \N/\B \\
\F & \T \\
\end{tabular}

\bigskip

Despite the apparent identity of K3 and LP (other than the mere change of names of \Na and B), the two logics are in fact different by virtue of their differing accounts of what semantic values are to be considered as privileged.  That is, which values are to be the \emph{designated values} for the logic --  those values that are semantically said to be the ones that the logic is concerned to manifest in a positive light.  In both logics this positivity results from the value(s) that exhibit (at least some degree of) truth: but in K3 only \Ta has that property whereas in LP both \Ta and \Ba manifest some degree of truth.  And therefore LP treats the semantically privileged values -- the \emph{designated values} -- to be both \Ta and \B, while in K3 the only designated value is \T.

Although both K3 and LP have no formula whose value is always \T, LP (unlike K3) does have formulas whose semantic value is always designated.  In fact, the class of propositional LP formulas that are designated is identical to that of propositional classical 2-valued logic, a fact to which we return in  \S\ref{Wrong}.  

Note that these truth tables give the classical values for compound formulas with classical components, and give the intermediate value for formulas all of whose components have that value.  Thus neither logic is functionally complete: no truth function giving a non-classical (classical) value for uniformly classical (non-classical) arguments can be represented in them.

\section{FDE and M:  Four-Valued Logics}\label{FDE}
The logic FDE was described in \citep{Belnap92} and \citep{Dunn76}.  It is a four-valued system: the values are \T, \F, \B, and \N\ldots the four values we have already encountered\footnote{Or, returning to the conception of a relational assignment, FDE allows any subset of the two ``real'' truth values:  $\{\T\}, \{\F\}, \{\T,\F\}$, and $\emptyset$.} (although the intuitive semantics behind these values, as given by Belnap, are perhaps somewhat different than we encountered above).  
K3 and LP agree with each other (and with classical logic) as regards the values \Ta and \F, so in combining them it is easy to identify K3's \Ta and \Fa with LP's \Ta and \F.  FDE then agrees with K3 about \N: agrees that it is not a designated value, and agrees on the values of combinations of \Na with \Ta and \F.  FDE also agrees with LP about \B: agrees that it is a designated value, and agrees on the values of combinations of \Ba with \Ta and \F.  Neither of the three-valued logics allow combinations of \Ba with \N, so these combinations need a bit more thought.  
The choice made is perhaps most easily described by representing the four values as \emph{sets} of the two classical values: 
\T=\{T\}, \B=\{T,F\}, \N=\{~\}, \F=\{F\} 
(where of course the bold-face values stand for FDE values and the normal text values stand for classical ones!).  A given classical value then goes into the set of values of a compound formula just in case the component formulas have values containing classical values which, by the classical rules, would yield the given value for the compound.  (This way of thinking also makes the designation statue of \Ba and \Na seem natural: an FDE value is designated just in case it, thought of as a set, contains the classical value T.)

This treatment of the four values and their ordering is perhaps best visualized by the diagram:

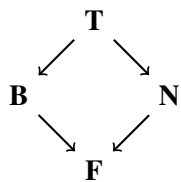
\hspace{3cm}
\begin{tikzpicture}
\node (I)    at ( 2,2)   {\T};
\node (II)   at (1,1)   {\B};
\node (III)  at (3,1) {\N};
\node (IV)   at (2,0) {\F};

\draw[thick,->] (I) -- (II);
 \draw[thick,->]      (I) -- (III);
 \draw[thick,->]      (II) -- (IV);
  \draw[thick,->]      (III) -- (IV);
 \end{tikzpicture}

\noindent The four values form a lattice, with the value of a conjunction (disjunction) being the meet (join) of the values of its conjuncts (disjuncts), and with negation interpreted as inverting the lattice order.

More formally, we take the basic truth-tables for $\land,\lor$ and $\neg$ to be straightforwardly taken from K3 and LP, for the values that do not involve only \Ba and \Na  together.  For those values we interpolate as given in the truth tables in Table \ref{FDEtable} for the FDE connectives.

\begin{table}[!h]
\begin{center}
\begin{tabular}{c c | c c c}
$\varphi$ & $\psi$ & $\neg\varphi$ & $(\varphi\land\psi)$ & $(\varphi\lor\psi)$ \\
\hline
\T & \T & \F & \T & \T \\
\T & \B & & \B & \T \\
\T & \N & & \N & \T \\
\T & \F & & \F & \T \\
\B & \T & \B & \B & \T \\
\B & \B & & \B & \B \\
\B & \N & & \F & \T \\
\B & \F & & \F & \B \\
\N & \T & \N & \N & \T \\
\N & \B & & \F & \T \\
\N & \N & & \N & \N \\
\N & \F & & \F & \N \\
\F & \T & \T & \F & \T \\
\F & \B & & \F & \B \\
\F & \N & & \F & \N \\
\F & \F & & \F & \F \\
\end{tabular}
\caption{Four-valued truth-tables for the basic connectives of FDE}
\label{FDEtable}
\end{center}
\end{table}

\section{What's Wrong with K3, LP, and FDE}\label{Wrong}
We start with K3 and LP.  Note that the $\land,\lor,\neg$ truth-functions always yield a \Na (or a \B) whenever all the input values are \Na (or \B, respectively).  This means that there are no formulas that always take the value \T.  In these logics, those three truth functions are all of the primitive functions, and so that same fact also implies that not every truth function is definable in K3 (or LP).  And if $\varphi$'s being semantically valid means that $\varphi$ always takes the value \Ta for all input values, as it does in K3, then there are no semantically valid formulas in K3.

And given that these are all the primitive truth functions, then the only plausible candidate for being a conditional in these logics comes by way of the classical definition:
$ (\varphi\supset\psi) =_{df} (\neg\varphi\lor\psi) $.

\begin{table}[!h]\label{K3LParrows}
\begin{center}
\begin{tabular}{c c| c c}
$\varphi$ & $\psi$ & ($\varphi$ $\supset_{K3}$ $\psi$) & ($\varphi$ $\supset_{LP}$ $\psi$) \\
\hline
\T & \T & \T & \T \\
\T & \N/\B & \N & \B \\
\T & \F & \F & \F \\
\N/\B & \T & \T & \T \\
\N/\B & \N/\B & \N & \B \\
\N/\B & \F & \N & \B \\
\F & \T & \T & \T \\
\F & \N/\B & \T & \T \\
\F & \F & \T & \T \\
\end{tabular}
\caption{Conditionals available in K3 and LP, using the classical definition}
\end{center}
\end{table}

The only difference between $\supset_{K3}$ and $\supset_{LP}$ is that in K3 the ``third value'' is understood as indicating a truth-value gap, i.e., as \emph{neither} \Ta nor \F, whereas whereas in LP it is understood as indicating a truth-value glut, i.e., as \emph{both} \Ta and \F.  

One notes that the rule of inference, Modus Ponens (MP), $(\varphi\supset_{K3}\psi), \varphi \vDash \psi$, is a valid rule in K3:  if both premises are \T, then the conclusion will also be, as can be seen from the truth table for $\supset_{K3}$.  Despite this, the \emph{statement} of MP, $((\varphi\supset_{K3}\psi)\land \varphi) \supset_{K3}  \psi$, does \emph{not} always exhibit the value \Ta (since no formula of K3 always takes the value \T).
Thus the deduction theorem does not hold for K3.
Matters are the reverse for the logic LP: the rule of MP is \emph{invalid}, as can be seen by making $\varphi$ be \Ba and $\psi$ be \F.  In that case both $(\varphi\supset_{LP}\psi)$ and $\varphi$ have designated values (both are \B), but yet the conclusion, $\psi$, is \F.  On the other side, the \emph{statement} of MP, $((\varphi\supset_{LP}\psi)\land \varphi) \supset_{LP}  \psi$, is always \T.  (One can see that the counterexample to the rule MP makes both conjuncts of the antecedent be \Ba and hence that antecedent is \B.  But a \Ba antecedent with a \Fa consequent is evaluated \B.  Hence the statement is designated.)  

As we remarked above, the theorems of LP are identical to those of classical logic --
LP does nothing more than divide up the classical notion of truth into two
parts: the ``true-only'' and the ``true-and-also-false''.
As both types of truth are designated values in the logic, the theorems of
the two are the same, and hence LP is not different from classical logic so
far as logical truth goes.
Where it does differ is in the class of valid rules of inference, as we have noted.  The reverse describes the situation in K3: whereas the deduction theorem holds and modus ponens fails for LP, in K3 the deduction theorem fails but modus ponens holds.

So, it seems that both K3 and LP are not very useful as guides to reasoning, despite their apparent (to some) virtues in accounting for intuitive semantic values of sentences that exhibit some troublesome features (such as vagueness or semantic paradox).  After discussing the logical state of K3 and LP, and the disappointing properties of the conditional operators available in them, Soloman Feferman put it:
\begin{quote}
Multiplying such examples, I conclude that \emph{nothing like 
   sustained ordinary reasoning can be carried on in either 
   logic.}  (\citealp[p.~264]{Feferman84}, italics in original.)
\end{quote}
On the topic of non-classical logics more generally, \citet{vanFraassen69} expresses the worry in more picturesque terminology.  New logics, he worries, along with 
 ``the appearance of wonderful new `logical' connectives, and of rules of `deduction' resembling the prescriptions to be read in \emph{The Key of Solomon}'' will  make ``standard logic texts read like witches' grimoires'', and will ``incline one to dismiss the technical study of [non-classical logic] as a mathematical parlor game.''

Turning now to FDE, it is clearly a merger of K3 and LP, so shares all their shortcomings.  
Since FDE includes the feature of K3 that there are no truth functions from all-atomics-having-the-value-\Na to any other truth value, it follows that there are  no formulas that always take one or the other of the designated values \Ta or \B.  And as with K3, the logic is functionally incomplete since there is no formula that will yield one of \T, \B, or \Fa if all its atomic letters are assigned \N.    Like K3 and LP, FDE does not have a usable conditional and hence there is no sense in which it is a logic that one can reason with.  
Although \cite{Feferman84} didn't include FDE in his disparaging remarks we just quoted, since FDE is simply a ``gluing together '' of K3 and LP -- the two logics Feferman did complain about -- it is clear that he would have held the same opinion about FDE.   

We wish to show that matters are not so dire as Feferman seems to think for the FDE-related logics we are investigating.  And so we now turn to possible ways to fix these shortcomings.

\section{An FDE Fix?}\label{FDEfix}

Various researchers have added conditionals to K3 and LP, with the thought of being able to do ``sustained reasoning'' in these logics.  We might mention here \L ukaseiewicz's conditional added to K3\footnote{Although of course this is not the historical reason he came up with his conditional, since he didn't have K3 before him at the time.}, with the intent of accommodating ``future contingents.''  This yields a logic that is usually called \L3.  And conditionals can be added to LP with the intent of allowing reasoning to take place in the context of dialetheism.  We will look here at the more closely at the logics RM3 and a logic we will call LP+cmi that are generated by adding different conditionals to LP.  

Although the metamathematical study of RM3 has been fruitful in the study of relevance logics, it seems to us that LP+cmi is likely to be a more convenient to \emph{use}\footnote {Tedder uses the name A3 for what we call LP+cmi.  Since we are comparing a number of logics, we prefer a more systematic nomenclature.} -- particularly in light of Tedder's \citeyear{Tedder2015} employment of it in formulating mathematically interesting axiomatic theories.   Somewhat to our surprise, it generalizes to the 4-valued case!

Here is the conditional we propose: we call it the ``classical material implication'' (or ``cmi'') and we symbolize it $\cmi$. 

\begin{table}[htb]
\begin{center}
\begin{tabular}{c | c c c c}
$\cmi$ & T & B & N & F \\
\hline
T & T & B & N & F \\
B & T & B & N & F \\
N & T & T & T & T \\
F & T & T & T & T \\
\end {tabular}
\caption{Truth matrix for $\cmi$.}
\label{defcmi}
\end{center}
\end{table}

In \S\ref{FDE} we remarked that the natural understanding of FDE's semantic values is that the designated ones are \Ta and \B, while the undesignated ones are \Na and \F.  As can be seen in Table \ref{defcmi}, any conditional that takes a designated value has the feature that if its antecedent is also designated, then its consequent will be also.  It can also be seen that if the values are restricted to the classical \Ta and \F, then the conditional mirrors the classical $\supset$.  We also obtain the truth table for the classical $\supset$ if we look at this with blurred vision, so that the two designated values blur into one and the two undesignated values similarly coalesce, showing that the valid formulas (i.e., those taking a designated value on every assignment of truth values to their atoms) and inferences (i.e., those preserving designation on every assignment) are precisely those of the pure $\supset$ fragment of classical logic.  Combined with a similarly blurred view of the truth tables for $\land$ and $\lor$ we can extend this observation to formulas having these connectives as well as $\rightarrow_{cmi}$, and we can extend it further to quantifiers if we think of them as generalized conjunctions and disjunctions.  Thus:
\begin{proposition}
For \emph{positive} formulas (i.e., those not containing $\neg$), the valid formulas and inferences of FDE+cmi are exactly those of classical logic.
\end{proposition}

The 4-valued logic FDE+cmi has two obvious 3-valued extensions, defined semantically by restricting the set of truth values allowed: K3+cmi defined by reference to \{\T,\N,\F\} and LP+cmi defined by reference to \{\T,\B,\F\}.  These logics result from the addition of a conditional connective, $\rightarrow_{cmi}$, defined by the relevant rows of the matrix of Table \ref{defcmi}, to the conditional-free logics K3 and LP.

\subsection{FDE+cmi}
The classical nature of the conditional in FDE+cmi makes it a much more pleasant item to work with than most previously proposed conditionals in many-valued logics: something like sustained ordinary reasoning \emph{can} be carried out in FDE+cmi and in its 3-valued extensions.  However, although the positive logic is completely classical, there are some surprises in the interaction of $\rightarrow$ (we will henceforth leave the subscript off this operator) with negation.  One unpleasant one is that the principle of Contraposition fails:  if $\sem{\varphi}=\T$ and $\sem{\psi}=\B$, then $\sem{\varphi\rightarrow\psi}$ takes the designated value \B, but $\sem{\neg\psi\rightarrow\neg\varphi}$ is \F.  And also,  if $\sem{\varphi}=\N$ and $\sem{\psi}=\F$, then $\sem{\varphi\rightarrow\psi}$=\T, but $\sem{\neg\psi\rightarrow\neg\varphi}$= \F. Since a counterexample can be obtained using either of the nonclassical values, Contraposition is invalid not only in FDE+cmi but also in both of its 3-valued extensions K3+cmi and LP+cmi.

We can, of course, define a new conditional connective for which Contraposition holds:

\medskip
$ (\varphi\Rightarrow\psi) =_{df} ((\varphi\rightarrow\psi)\land(\neg\psi\rightarrow\neg\varphi))$

\medskip
\noindent This is a useful connective!  The corresponding biconditional

\medskip
$(\varphi\Leftrightarrow\psi) =_{df} ((\varphi\Rightarrow\psi)\land(\psi\Rightarrow\varphi))$

\medskip
\noindent
takes a designated value if and only if $\sem{\varphi}=\sem{\psi}$.  
(Note that $\sem{(\varphi\Leftrightarrow\psi)}$=\Ba if $\sem{\varphi} = \sem{\psi}$ = \B, and =\Ta if $\sem{\varphi}$ and $\sem{\psi}$ have one of the other three values.)  
As a result, it supports a principle of Substitution:  formulas of the form

\medskip
$(\varphi\Leftrightarrow\psi) \rightarrow ((\cdots\varphi\cdots)\Leftrightarrow(\cdots\psi\cdots))$

\medskip
\noindent are valid.  (In contrast, the biconditional similarly defined in terms of our basic conditional yields a designated value just in case either both of its terms have designated values or both have undesignated values.  As a result, it does not license substitution:  If $\varphi$ and $\psi$ have \emph{different} designated values, and also if they have different undesignated values, $(\varphi\leftrightarrow\psi)$ will have a designated value, but $(\neg\varphi\leftrightarrow\neg\psi)$ will not.  On the other hand, $\leftrightarrow$ can be added to the list of ``positive'' connectives, $\land, \lor, \rightarrow$, whose logic is exactly classical.)

On the other hand, the $\Rightarrow$ connective has certain undesirable features which militate against its adoption as the basic conditional operator of a logic designed for use.  Principles analogous to some of the \emph{structural} rules of \citep{Gentzen34}, easily derivable by the conventional natural deduction rules for the conditional, fail for it.  Some also fail for $\Rightarrow$ in K3+cmi, and others in LP+cmi.

\subsection{On the K3 Side}\label{K3side}
The truth tables for $\rightarrow$ and $\Rightarrow$ differ on only one of the nine lines for K3, as indicated in Table \ref{K3arrows}: if $\sem{\varphi}$=\Na and $\sem{\psi}$=\F, then $\sem{\varphi\rightarrow\psi}$=\Ta but $\sem{\varphi\Rightarrow\psi}$=\N.  

\begin{table}[htb]
\begin{center}
\begin{tabular}{cc|ccr}
$\varphi$ & $\psi$  & ($\varphi\rightarrow\psi$)  & ($\varphi\Rightarrow\psi$) &   \\
\hline
\T  & \T  & \T   & \T &\\
\T  & \N  & \N    & \N &   \\
\T  & \F  & \F   & \F  & \\
\N  & \T  & \T   & \T &  \\
\N  & \N  & \T    & \T &  \\
\N  & \F  & \T    & \N & $\Leftarrow$ \\
\F  & \T  & \T    & \T  & \\
\F  & \N  & \T   & \T &  \\
\F  & \F  & \T   & \T  &  \\
\end{tabular}
\caption{Comparison of two conditionals in K3+cmi}
\label{K3arrows}
\end{center}
\end{table}

\medskip
\noindent This is enough, however, to invalidate the principle of Contraction:

\medskip
$(\varphi\Rightarrow(\varphi\Rightarrow\psi))\Rightarrow(\varphi\Rightarrow\psi)$

\medskip
\noindent This should be old news! The derived truth table for $\Rightarrow$ in K3+cmi is exactly that of the conditional of \L ukasiewicz's 3-valued logic, for which Contraction failure is familiar.  Since the interpretations of the other connectives are the same for K3 and \L3, we have

\begin{proposition}
\L ukasiewicz's 3-valued logic is faithfully interpretable in K3+cmi.
\end{proposition}
\noindent In fact, we have the converse as well\footnote{ \citep{Nelson59} should be credited with the observations of the failure of contraposition, its brute force restoration, and contraction failure in his system of constructible falsity\ldots which is essentially the addition of intuitionistic implication to K3.  What our observation adds to this is that these phenomena do not depend on the intuitionistic nature of Nelson's implication, but arise already in the 3-valued and 4-valued logics.}: our $\rightarrow$ can be defined in \L3 by

\medskip
$(\varphi\rightarrow\psi) =_{df} (\varphi\Rightarrow(\varphi\Rightarrow\psi))$,

\medskip
\noindent so K3+cmi is faithfully interpretable in \L3.  K3+cmi and \L3 can thus be seen as, in effect, alternative formulations of a single logic (see further discussion in Section \ref{synonymous}).  We think the classical nature of $\rightarrow$ makes it easier to use than $\Rightarrow$, and recommend translation into K3+cmi to anyone interested in proving theorems in \L3.

\subsection{On the LP Side}\label{LPside}
Now looking at the logic with truth values \{\T,\B,\F\}, we see in Table \ref{LParrows} that $\Rightarrow$ again differs from $\rightarrow$ on only one of the nine lines: but this time if $\sem{\varphi}$=\Ta and $\sem{\psi}$=\B, then $\sem{\varphi\rightarrow\psi}$=\Ba but $\sem{\varphi\Rightarrow\psi}$=\F.  

\begin{table}[htb]
\begin{center}
\begin{tabular}{cc|ccr}
$\varphi$ & $\psi$  & ($\varphi\rightarrow\psi$)  & ($\varphi\Rightarrow\psi$) &   \\
\hline
\T  & \T  & \T   & \T &\\
\T  & \B  & \B    & \F &  $\Leftarrow$ \\
\T  & \F  & \F   & \F  & \\
\B  & \T  & \T   & \T &  \\
\B  & \B  & \B    & \B &  \\
\B  & \F  & \F    & \F &  \\
\F  & \T  & \T    & \T  & \\
\F  & \B  & \T   & \T &  \\
\F  & \F  & \T   & \T  &  \\
\end{tabular}
\caption{Comparison of two conditionals in LP}
\label{LParrows}
\end{center}
\end{table}

This difference in truth table is enough to invalidate the principle of Thinning:

\medskip
\hspace{.5cm}$\varphi\Rightarrow(\psi\Rightarrow\varphi)$

\medskip
\noindent If $\sem{\varphi}=\Ba$ and $\sem{\psi}=\T$, then $\sem{\psi\Rightarrow\varphi}$ and hence $\sem{\varphi\Rightarrow(\psi\Rightarrow\varphi)}$ will both be \F.  As \cite{Tedder2015} notes, $\Rightarrow$ in LP+cmi has exactly the truth table of the logic RM3 (and, since RM3 is a cousin of the relevance family of logics, failure of Thinning in it is just what one would expect).  RM3 and LP agree on $\{\land,\lor,\neg\}$, so we have 

\begin{proposition}
RM3 can be faithfully interpreted in LP+cmi.
\end{proposition}

\noindent Again, the converse is also true: $\rightarrow$ can be defined in terms of $\Rightarrow$ in RM3 by

\medskip
\hspace{.5cm}$(\varphi\rightarrow\psi) =_{df} ((\varphi\Rightarrow\psi)\lor\psi)$

\medskip
\noindent So in the same sense, RM3 and LP+cmi can be thought of as alternative formulations of a single logic (see discussion in Section \ref{synonymous}), and again, we think adoption of $\rightarrow$ as primitive is likely to be more convenient.

\subsection{About M}
Given that K3 and LP are obtained semantically from FDE by adding different conditions on valuations, and syntactically by adding different rules, one might incautiously conjecture that the FDE consequence relation is simply the intersection of the K3 and LP relations.  Not so: $(\varphi\land\neg\varphi)$ implies $(\psi\lor\neg\psi)$ in both 3-valued logics, but not in FDE.  There is thus a fifth logic in the neighbourhood, which is sometimes called \emph{Mingle} and which we will unimaginatively call M.

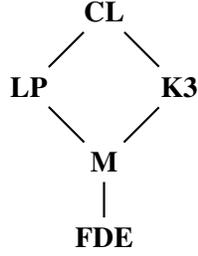
\hspace{4cm}
\begin{tikzpicture}
\node (I)    at ( 2,3)   {{\bf CL}};
\node (II)   at (1,2)   {{\bf LP}};
\node (III)  at (3,2) {{\bf K3}};
\node (IV)   at (2,1) {{\bf M}};
\node (V) at (2,0)  {{\bf FDE}};

 \draw[thick,-] (I) -- (II);
 \draw[thick,-]      (I) -- (III);
 \draw[thick,-]      (II) -- (IV);
  \draw[thick,-]      (III) -- (IV);
  \draw[thick,-]    (V)--(IV);
 \end{tikzpicture}

 It is characterized semantically as the set of inferences preserving designation on every FDE valuation which does not assign \Ba to one formula and \Na to another, and syntactically by adding the above-mentioned inference as a general rule to a formulation of FDE.  We think M seems like a rather silly logic (why should the presence of a single truth-value glut rule out the existence of any truth-value gaps, or vice versa?), but record here that M+cmi can be formulated in the obvious way, and that our completeness proof in the Appendix for FDE+cmi extends without difficulty to cover it.

\subsection{Putting them Together}
We have not explored the behaviour of $\Rightarrow$ in FDE+cmi, and, in contrast to the situation with the 3-valued logics, do not know of a historically proposed equivalent for it.  Contraction and Thinning will, of course, both fail for $\Rightarrow$ in the 4-valued logic.  With no relevant insight, we have resorted to construction of truth tables to verify that $\rightarrow$ can be recovered from $\Rightarrow$ in the 4-valued environment by the double-barrelled definition.

\medskip
$(\varphi\rightarrow\psi) =_{df} ((\varphi\Rightarrow(\varphi\Rightarrow\psi)) \lor \psi)$

\section{More and Less Drastic Expansions}
In \cite{HP17} it was shown that a Second Order logic based on LP was surprisingly weak.  This was due to the limited expressive power of the language with no conditional operator.  In contrast, FDE+cmi, with its classical conditional, can be shown to contain full classical Second Order Logic.  Using propositional quantification, we can define a \emph{falsum} propositional constant:

\begin{quote} {\bf f} $= \forall p(p)$\end{quote}
(Obviously set/property quantification would do as well: $\forall X\forall x(X(x))$.)  This constant and the conditional give us in effect a ``classical'' negation: the conditional 
$(\varphi\rightarrow$ {\bf f}) 
takes the value \Fa when $\sem{\varphi}$ is designated and the \Ta when $\sem{\varphi}$ is undesignated.  Call a predicate (monadic or relational) \emph{classical} just in case no atomic formula in it, on any assignment to the individual variables, takes either of the intermediate truth values \Na or \B.  \Na can be ruled out by the First Order formula $\forall x(\psi(x)\lor\neg\psi(x))$,  which will not take a designated value if the predicate yields the value \Na for any individual in the domain.  The possibility that $\psi$ somewhere yields the value \B, however, cannot be excluded by a purely First Order formula: if $\psi$ yields \Ba for every individual, then every First Order formula in which $\psi$ is the only bit of non-logical vocabulary occurring will have the designated value \B.  Using the ersatz classical negation, however, we can say $((\exists x(\psi(x)\land\neg\psi(x))\rightarrow$ {\bf f}).  The classicality of the predicate, then, can by expressed by the conjunction of these two formulas.  We may then interpret classical Second Order Logic in Second Order FDE+cmi by simply restricting all Second Order quantifiers to classical predicates.

Coming back down to the fragment of Second Order FDE+cmi with only propositional quantification, we can define the other three propositional constants.  \emph{Verum} is easy:
\begin{quote} {\bf t} $= \exists p(p)$\end{quote}
Defining a constant for \Ba is perhaps less obvious.  $\exists p(p\land\neg p)$ would work in the 3-valued logic LP, but when \Na is also available it fails: $\sem{p}=\Ba$ gives $\sem{p\land\neg p}=\Ba$, $\sem{p}=\Na$ gives $\sem{p\land\neg p}=\N$, and the existential quantification will have as value the join of these in the truth-value lattice: \T!\footnote{The existentially quantified formula takes the join of the values \N, \B, and \F\ldots which, surprisingly, is \T.  One doesn't expect a disjunction to take a higher value than any of its disjuncts, but in this case, because the four values are not linearly ordered, it does.} However, $\sem{p\rightarrow p}$=\Ba when $\sem{p}$=\B, and is \Ta otherwise, so we may define 
\begin{quote} {\bf b} $= \forall p(p\rightarrow p)$\end{quote}

Defining a constant for \Na in FDE+cmi is altogether harder (though  $\exists p(p\land\neg p)$ would work in the 3-valued logic K3+cmi).  Indeed, it can be shown that no \emph{definiens} with, in prenex form, a single block of propositional quantifiers (all universal or all existential) will work.   
\begin{proof}
Consider a purely propositional formula of FDE+cmi.  An assignment giving the value \Ba to all of its variables will give the formula the value \B.  Thus the set of values assumed by the formula on different assignments to its propositional variables will include \B.  The value of the sentence formed by binding its variables by existential quantifiers will be the lattice join of the values in this set, and so must be either \Ba or \T.   The value of the sentence formed by binding its variables by universal quantifiers will be the lattice meet of the values in this set, and so must be either \Ba or \F.  In neither case will the quantified formula serve as a definiens for the constant {\bf n}.
\end{proof}

The propositional constant {\bf n} can, however, be defined by a sentence of more complicated quantificational structure.  
$\sem{(\varphi\Leftrightarrow\psi)}$ is designated if and only if $\sem{\varphi} = \sem{\psi}$.  
It can thus be thought of as expressing identity of truth value.  The conditional $((\varphi\Leftrightarrow$ {\bf t}$)\rightarrow$ {\bf f}$)$, then, says that the value of $\varphi$ is \emph{not} \T, and similarly for $((\varphi\Leftrightarrow$ {\bf b}$)\rightarrow$ {\bf f}$)$ and $((\varphi\Leftrightarrow$ {\bf f}$)\rightarrow$ {\bf f}$)$.  So we may define {\bf n} by
\begin{quote}{\bf n} $=_{df} \exists q(((q\Leftrightarrow${\bf ~t}$)\rightarrow${\bf ~f}$)\land ((q\Leftrightarrow${\bf ~b}$)\rightarrow${\bf ~f}$)\land((q\Leftrightarrow${\bf ~f}$)\rightarrow${\bf ~f}$) \land q)$\end{quote}
For, if $q$ in the matrix is assigned one of the values \T, \B, or \F, one of the first three conjuncts, and so the whole, will have the value \F.  If $\sem{q}$=\N, however, the first three conjuncts will all have the value \T, but the fourth,  $q$, and so the whole conjunction, will have the value \N.  Since the join of \Fa and \Na is \N, the value of the full, quantified, sentence is \N.  By extracting the quantifiers concealed in the propositional constants in the contained biconditionals in the right order, we can put the definiens into prenex form with only one alternation of quantifiers (so, it will be $\Sigma_2$).

A final observation.  Propositional FDE+cmi, unaugmented, is not functionally complete: no 4-valued truth function mapping uniformly classical arguments to an intermediate value, or mapping uniformly \Ba arguments to a value other than \B, can be represented.  Adding the four propositional constants (now thought of as primitives) gives us functional completeness, for an arbitrary $n$-ary truth function can be expressed by the conjunction of $4^n$ conditionals, each having as antecedent an $n$-place conjunction of $\Leftrightarrow$ biconditionals specifying a set of values for the arguments, and as consequent the constant for the desired value.  That is, we can say in the language that a formula takes a particular one of the semantic values, thus mimicking the ``parametric'' J-operators of \citep{RT52}, and together with the defined truth-constants, we can construct DNF or CNF formulas that manifest any 4-valued truth table by using the technique of \citep{RT52}.

\section{Synonymous Logics}\label{synonymous}

In \citet{Pell84,PU03}, the notion of ``translational equivalence'' was introduced and defined, and used to describe a concept of \emph{synonymous logics}.  This concept was intended to describe cases where two logic systems were ``really the same system'' despite having different formulations, different vocabulary, and possibly having such different formulation that it would not be at all obvious that the logics were ``really the same.''  This notion was shown to be different from various other conceptions in the literature, such as \emph{mutual interpretability} and having \emph{exact translations between logics}, which were shown to be weaker; other notions, such as having identical definitional extensions, were shown to be the same conception.\footnote{Further aspects of the notion, as well as formal details, are in \citet{PU03}.}  

Two logics, $\mathcal{L}_1$ and $\mathcal{L}_2$, are translationally equivalent if and only if there are translation schemes $t_1$ from $\mathcal{L}_1$ into $\mathcal{L}_2$ and $t_2$ from $\mathcal{L}_2$ into $\mathcal{L}_1$ such that 
\begin{enumerate}
\item if $\vdash_{\mathcal{L}_1}\varphi$ then $\vdash_{\mathcal{L}_2}\varphi^{t_1}$
\item if $\vdash_{\mathcal{L}_2}\varphi$ then $\vdash_{\mathcal{L}_1}\varphi^{t_2}$
\item for any formula $\varphi$ in $\mathcal{L}_1,\; (\varphi^{t_1})^{t_2}$ is equivalent to $\varphi$ (in  $\mathcal{L}_1$)
\item for any formula $\varphi$ in $\mathcal{L}_2,\; (\varphi^{t_2})^{t_1}$ is equivalent to $\varphi$ (in  $\mathcal{L}_2$)\footnote{In \citet{PU03} it was assumed that the logics in question had a ``biconditional equivalence connective'' and the third and fourth conditions were expressed in terms of the biconditional being a theorem in the appropriate logics.  In the context of that paper, the logics were classical except for modal operators, and so there were such equivalence operators in each logic.  In the present context, we cannot assume that the biconditionals of the various logics will operate in the same way, and so we envisage checking the ``equivalent to'' conditions semantically, by simply looking at the relevant truth tables.}
\end{enumerate}

We are in a position to show some new results of synonymity of logics.  Since the languages of our different logics are identical except for their (bi)conditionals, we will employ the following translations for all the portions of the logics involved except the differing (bi)conditionals in the logics:

\begin{itemize}
\item For $\varphi$ an atomic sentence, $(\varphi)^t$ is $\varphi$
\item For negated formulas,  $(\neg\varphi)^t$ is $\neg(\varphi)^t$ 
\item If $\circ$ is any binary operator other than a conditional, $(\varphi\circ\psi)^t$ is $((\varphi)^t\circ(\psi)^t)$
\end{itemize}

In \S\ref{K3side} we showed that systems K3+cmi and \L3 could be faithfully interpreted in each other.  We can now prove a stronger result.

\begin{theorem}[\bf{K3+cmi and \L3 are synonymous logics}]
Let $(\varphi\rightarrow_{K3+cmi}\psi)^{t_1}$ be $(\varphi\rightarrow_{\L3}(\varphi^{t_1}\rightarrow_{\L3}\psi^{t_1}))$, and
let $(\varphi\rightarrow_{\L3}\psi)^{t_2}$ be $(\varphi\Rightarrow_{K3+cmi}\psi))$. \ldots i.e., $[(\varphi\rightarrow_{K3+cmi}\psi)\land(\neg\psi\rightarrow_{K3+cmi}\neg\varphi)]$.

Then
the four conditions for translational equivalence are met: the first two by the facts of identity of translations for all connectives except conditionals, and for conditionals by the truth tables for $\rightarrow_{\L3}$ and $\rightarrow_{K3+cmi}$.  The second two are met because the stated translations $t_1$ and $t_2$ of this theorem ``cycle'' -- that is, they each immediately introduce a conditional formula in the other logic which in turn will be eliminated when translated back into the first logic.  
\end{theorem}
\noindent Note that $((\varphi\rightarrow_{K3+cmi}\psi)^{t_1})^{t_2}$ = $(\varphi\Rightarrow_{K3+cmi}(\varphi\Rightarrow_{K3+cmi}\psi))$, which -- in primitive notation -- is 

$[\varphi\rightarrow_{K3+cmi}((\varphi\rightarrow_{K3+cmi}\psi)\land(\neg\psi\rightarrow_{K3+cmi}\neg\varphi))]\;\land\;$

\hspace{.5cm}$[\neg((\varphi\rightarrow_{K3+cmi}\psi) \land (\neg\psi\rightarrow_{K3+cmi}\neg\varphi))\rightarrow_{K3+cmi}\neg\varphi]$

\noindent  This last formula and $(\varphi\rightarrow_{K3+cmi}\psi)$ can be checked using  the 9-row K3+cmi truth tables to show their equivalency.

\medskip
\noindent Note also that $((\varphi\rightarrow_{\L3}\psi)^{t_2})^{t_1}$ = $[\varphi\rightarrow_{\L3}(\varphi\rightarrow_{\L3}\psi)]\;\land\;[\neg\psi\rightarrow_{\L3}(\neg\psi\rightarrow_{\L3}\neg\varphi)]$
which can also be checked in the 9-row truth tables for \L3 to demonstrate its equivalence to $(\varphi\rightarrow_{\L3}\psi)$.

In \S\ref{LPside} we showed that systems LP+cmi and RM3 could be faithfully interpreted in each other.  We can now prove a stronger result.  

\begin{theorem}[\bf{LP+cmi and RM3 are synonymous logics}]

Let $(\varphi\rightarrow_{LP+cmi}\psi)^{t_1}$ be $((\varphi\rightarrow_{RM3}\psi) \lor \psi)$, and
let $(\varphi\rightarrow_{RM3}\psi)^{t_2}$ be $(\varphi\Rightarrow_{LP+cmi}\psi))$. \ldots i.e., $[(\varphi\rightarrow_{LP+cmi}\psi)\land(\neg\psi\rightarrow_{LP+cmi}\neg\varphi)]$.

Then the four conditions for translational equivalence are met, in pretty much the same way as they were met in the previous theorem, keeping in mind that $\rightarrow_{K3+cmi}$ and $\rightarrow_{LP+cmi}$ are different, as are $\Rightarrow_{K3+cmi}$ and $\Rightarrow_{LP+cmi}$.
\end{theorem}

These instances of synonymous logics strike us as both unexpected and also as ``cleaner'' versions of synonymy of logics than the one(s) displayed in \citep{PU03}, which employed a propositional constant in one of the logics.  As remarked in \cite{PU03}, two logics are translationally equivalent in this way if and only if they have a common definitional extension.  Note then for each of our pairs of synonymous logics, the appropriate 3-valued logic with two conditional operators,  $\rightarrow_{cmi}$ and $\Rightarrow_{cmi}$ is a definitional extension of both members of the pair.  And thus the two logics are translationally equivalent to each other.

\section{Concluding Remarks}
We have discussed a family of logics that are related to FDE, showing how they are related to each other and also to some other logics that have populated the literature (such as \L3 and RM3).  Surprisingly, perhaps, we are also able to show some new and ``cleaner'' examples of synonymous logics (in the sense of \citealp{PU03}).   Deductive systems for the various logics, as well as soundness and completeness proofs are in the Appendices.
\newpage
\bibliographystyle{chicago}
\bibliography{FDEbiblio}

\begin{thebibliography}{}

\bibitem[\protect\citeauthoryear{Alxatib and Pelletier}{Alxatib and
  Pelletier}{2011}]{AP11}
Alxatib, S. and F.~J. Pelletier (2011).
\newblock The psychology of vagueness: Borderline cases and contradictions.
\newblock {\em Mind and Language\/}~{\em 26}, 287--326.

\bibitem[\protect\citeauthoryear{Asenjo}{Asenjo}{1966}]{Asenjo66}
Asenjo, F. (1966).
\newblock A calculus of antinomies.
\newblock {\em Notre Dame Journal of Formal Logic\/}~{\em 7}, 103--105.

\bibitem[\protect\citeauthoryear{Beall and Colyvan}{Beall and
  Colyvan}{2001}]{BC01}
Beall, J. and M.~Colyvan (2001).
\newblock From heaps of gluts to hyde-ing the sorites.
\newblock {\em Mind\/}~{\em 110}, 401--408.

\bibitem[\protect\citeauthoryear{Belnap}{Belnap}{1992}]{Belnap92}
Belnap, N.D., J. (1992).
\newblock A useful four-valued logic: How a computer should think.
\newblock In A.~Anderson, J.~N.D.~Belnap, and J.~Dunn (Eds.), {\em Entailment:
  The Logic of Relevance and Necessity, Volume II}. Princeton UP.
\newblock First appeared as ``A Useful Four-valued Logic'' \emph{ Modern Use of
  Multiple-valued Logic} J.M. Dunn and G. Epstein (eds.), Dordrecht: D. Reidel,
  1977; and ``How a Computer Should Think'' \emph{Contemporary Aspects of
  Philosophy} G. Ryle (ed.), Oriel Press, 1977.

\bibitem[\protect\citeauthoryear{Cantini}{Cantini}{2014}]{Can14}
Cantini, A. (2014).
\newblock {Paradoxes and Contemporary Logic}.
\newblock In E.~Zalta (Ed.), {\em {The Stanford Encyclopedia of Philosophy}}.
\newblock
  {\url{http://plato.stanford.edu/archives/fall2014/entries/paradoxes-contemporary-logic/}}.

\bibitem[\protect\citeauthoryear{Dunn}{Dunn}{1976}]{Dunn76}
Dunn, J. (1976).
\newblock Intuitive semantics for first degree entailment and coupled trees.
\newblock {\em Philosophical Studies\/}~{\em 29}, 149--168.

\bibitem[\protect\citeauthoryear{Feferman}{Feferman}{1984}]{Feferman84}
Feferman, S. (1984).
\newblock Toward useful type-free theories, {I}.
\newblock {\em Journal of Symbolic Logic\/}~{\em 40}, 75--111.
\newblock Reprinted in Robert L. Martin (ed.) \emph{Recent Essays on Truth and
  the Liar Paradox}, pp.~237--288, Oxford: Clarendon Press. Page references to
  this reprinting.

\bibitem[\protect\citeauthoryear{Fitch}{Fitch}{1952}]{fitch52}
Fitch, F. (1952).
\newblock {\em Symbolic Logic: An Introduction}.
\newblock NY: Ronald Press.

\bibitem[\protect\citeauthoryear{Gentzen}{Gentzen}{1934}]{Gentzen34}
Gentzen, G. (1934).
\newblock Untersuchungen \"uber das logische {S}chlie\ss en, {I} and {II}.
\newblock {\em Mathematische Zeitschrift\/}~{\em 39}, 176--210, 405--431.
\newblock English translation ``Investigations into Logical Deduction''
  published in {\it American Philosophical Quarterly 1:} 288--306 (1964), and
  {\it2:} 204--218 (1965). Reprinted in M. E. Szabo (ed.) (1969) \emph{The
  Collected Papers of Gerhard Gentzen}, North-Holland, Amsterdam, pp. 68--131.
  Page references to the {\it APQ} version.

\bibitem[\protect\citeauthoryear{Hazen and Pelletier}{Hazen and
  Pelletier}{2017}]{HP17}
Hazen, A.~P. and F.~J. Pelletier (2017).
\newblock Second-order logic of paradox.
\newblock {\em Notre Dame Journal of Formal Logic\/}.
\newblock 2018 forthcoming.

\bibitem[\protect\citeauthoryear{Hyde}{Hyde}{1997}]{Hyde97}
Hyde, D. (1997).
\newblock From heaps and gaps to heaps of gluts.
\newblock {\em Mind\/}~{\em 106}, 641--660.

\bibitem[\protect\citeauthoryear{Hyde and Colyvan}{Hyde and
  Colyvan}{2008}]{HC08}
Hyde, D. and M.~Colyvan (2008).
\newblock Paraconsistent vagueness: Why not?
\newblock {\em Australasian Journal of Logic\/}~{\em 6}, 107--121.

\bibitem[\protect\citeauthoryear{Kleene}{Kleene}{1952}]{Kleene1952}
Kleene, S. (1952).
\newblock {\em Introduction to Metamathematics}.
\newblock Amsterdam: North-Holland.

\bibitem[\protect\citeauthoryear{Nelson}{Nelson}{1959}]{Nelson59}
Nelson, D. (1959).
\newblock Negation and separation of concepts in constructive systems.
\newblock In A.~Heyting (Ed.), {\em Constructively in Mathematics: Proceedings
  of the Colloquium Held at Amsterdam, 1957}, pp.\  208--225. Amsterdam: North
  Holland.

\bibitem[\protect\citeauthoryear{Pelletier}{Pelletier}{1984}]{Pell84}
Pelletier, F.~J. (1984).
\newblock Six problems in translational equivalence.
\newblock {\em Logique et Analyse\/}~{\em 108}, 423--434.

\bibitem[\protect\citeauthoryear{Pelletier and Urquhart}{Pelletier and
  Urquhart}{2003}]{PU03}
Pelletier, F.~J. and A.~Urquhart (2003).
\newblock Synonymous logics.
\newblock {\em Journal of Philosophical Logic\/}~{\em 32}, 259--285.
\newblock See also the authors' ``Synonymous Logics: A Correction'', \emph{JPL
  37}: 95--100, 2008.

\bibitem[\protect\citeauthoryear{Priest}{Priest}{1979}]{Priest79}
Priest, G. (1979).
\newblock The logic of paradox.
\newblock {\em Journal of Philosophical Logic\/}~{\em 8}, 219--241.

\bibitem[\protect\citeauthoryear{Priest}{Priest}{2006}]{Pri06}
Priest, G. (2006).
\newblock {\em {In Contradiction: A Study of the Transconsistent, 2nd Ed.}}
\newblock Oxford: Oxford University Press.

\bibitem[\protect\citeauthoryear{Priest}{Priest}{2018}]{Priest2018}
Priest, G. (2018).
\newblock Natural deduction systems for logics in the {FDE} family.
\newblock In H.~Wansing and H.~Omori (Eds.), {\em {FDE} Some 40 Years
  Later(?)}. Dordrecht: Springer.
\newblock Forthcoming.

\bibitem[\protect\citeauthoryear{Ripley}{Ripley}{2011}]{Ripley11}
Ripley, D. (2011).
\newblock Contradiction at the borders.
\newblock In {\em Vagueness in Communication}, pp.\  169--188. Dordrecht:
  Springer.

\bibitem[\protect\citeauthoryear{Rosser and Turquette}{Rosser and
  Turquette}{1952}]{RT52}
Rosser, J.~B. and A.~Turquette (1952).
\newblock {\em Many-Valued Logics}.
\newblock Amsterdam: North Holland.

\bibitem[\protect\citeauthoryear{Shramko and Wansing}{Shramko and
  Wansing}{2017}]{SW17truth-values}
Shramko, Y. and H.~Wansing (2017).
\newblock Truth values.
\newblock In E.~N. Zalta (Ed.), {\em The Stanford Encyclopedia of Philosophy\/}
  (Summer 2017 ed.). Metaphysics Research Lab, Stanford University.

\bibitem[\protect\citeauthoryear{Tedder}{Tedder}{2014}]{Tedder2014}
Tedder, A. (2014).
\newblock Paraconsistent logic for dialethic arithmetics.
\newblock Master's thesis, University of Alberta, Philosophy Department,
  Edmonton, Alberta, Canada.
\newblock Available at \url{https://www.library.ualberta.ca/catalog/6796277}.

\bibitem[\protect\citeauthoryear{Tedder}{Tedder}{2015}]{Tedder2015}
Tedder, A. (2015).
\newblock Axioms for finite collapse models of arithmetic.
\newblock {\em Review of Symbolic Logic\/}~{\em 8}, 529--539.

\bibitem[\protect\citeauthoryear{van Fraassen}{van
  Fraassen}{1969}]{vanFraassen69}
van Fraassen, B. (1969).
\newblock Presuppositions, supervaluations and free logic.
\newblock In K.~Lambert (Ed.), {\em The Logical Way of Doing Things}, pp.\
  67--92. New Haven: Yale UP.

\end{thebibliography}

\newpage

\section*{Appendix I:  Deductive Systems}
\cite{Tedder2014} contains a Hilbert-style axiomatic system, and (more interestingly) a (multiple succedent) sequent calculus for LP+cmi.  (\citet{Tedder2015} presents only the Hilbert system.)  The sequent calculus closely follows \citep{Gentzen34}'s system LP for classical logic, with the following changes:
\begin{enumerate}[(i)]
\item  Gentzen's rules for negation (``change the side and change the sign'') are dropped,
  \item double negation rules are added, allowing a sequent to be followed by one like it except that one of its formulas--on either side--is doubly negated,
\item  negative rules for conjunction and disjunction (and, analogously, for quantifiers) are added, allowing the insertion of a negated conjunction or disjunction under the same conditions as allow the insertion of its De Morgan equivalent disjunction or conjunction,
\item negative conditional rules are added, allowing $\neg(\varphi\rightarrow\psi)$ to be inserted under the same conditions as $(\varphi\land\neg\psi)$, and
\item Gentzen's identity axioms, $\varphi\vdash\varphi$,  are supplemented with ``Gap excluding'' axioms of the form $\vdash\varphi,\neg\varphi$.
\end{enumerate}

\noindent Cut elimination is proven by a straightforward adaptation of Gentzen's method.

Dropping the Gap exclusion axioms from this system yields a sequent calculus for FDE+cmi.  Replacing them with ``Glut excluding axioms'' of the form $\varphi,\neg\varphi\vdash$ gives one for K3+cmi, and replacing them instead with Mingle axioms of the form 
$\varphi,\neg\varphi\vdash\psi,\neg\psi$
 gives one for M+cmi.

    A natural deduction system, either in the style of Gentzen's NK or in the Fitch-style presentation of many American textbooks, for any of these logics will include
\begin{enumerate}[(i)]
\item the standard Introduction and Elimination rules for $\land,\lor,\rightarrow$ (and for the quantifiers in a First Order system),
\item Double Negation Introduction and Elimination rues, by which a formula and its double negation may each be inferred from the other,
\item Negative Introduction and Elimination rules for $\land$ and $\lor$ (and, analogously for the quantifiers), as in \citep{fitch52} enforcing the interdeducibility of negated conjunctions and disjunctions with their De Morgan equivalent disjunctions and conjunctions,
\item Negative Introduction and Elimination rules for $\rightarrow$, enforcing the interdeducibility of $\neg(\varphi\rightarrow\psi)$ with $(\varphi\land\neg\psi)$.
\end{enumerate}
As is well-known, the standard Introduction and Elimination rules for $\rightarrow$ give only the Intuitionistic logic of the conditional, so these have to be supplemented with some classicizing postulate to give the full classical logic of the positive connectives.  Addition of Peirce's Law, $(((\varphi\rightarrow\psi)\rightarrow\varphi)\rightarrow\varphi)$ as an axiom scheme is one conventional way of doing this, but an alternative axiom scheme, $(\varphi\lor(\varphi\rightarrow\psi))$ seems a good deal easier to work with, and can be converted into a moderately elegant natural deduction rule:

\begin{enumerate}[(v)]
\item A formula, $\chi$, may be asserted if it is derivable both from the hypothesis $\varphi$ and from the hypothesis $(\varphi\rightarrow\psi)$.
\end{enumerate}

    These rules suffice for the propositional logic FDE+cmi.  (A First Order system would have to supplement the standard and negative Introduction and Elimination Rules for the quantifiers with something to guarantee the ``constant domain'' inference 
    
\medskip
 $\forall x(\varphi\lor\Psi(x)) \vdash (\varphi\lor\forall x(\Psi(x))$
 
 \medskip
\noindent(see \citealt[\S21.31]{fitch52}.)    Systems for the other logics are obtained by adding rules to this basic system:
\begin{enumerate}[1)]
\item Ex falso quodlibet (``explosion''), for K3+cmi; 
\item Excluded middle: $\chi$ may be asserted if it is derivable both from the hypothesis $\varphi$ and from the hypothesis $\neg\varphi$, for LP+cmi; 
\item Mingle: $(\psi\lor\neg\psi)$ may be inferred from $(\varphi\land\neg\varphi)$,  for M+cmi.
\end{enumerate}

More visually: Natural deduction rules for FDE+cmi will be double negation and both the positive and negative IntElim rules for $\land$ and $\lor$.    Additionally, there is a series of rules for our $\rightarrow$ operator.

\bigskip
\noindent Double Negation:

$\begin{fitch}
\fh A  \\
\fa \neg\neg A  & $\neg\neg$ Int 
\end{fitch}$
\hspace{2cm}
$\begin{fitch}
\fh \neg\neg A \\
\fa A & $\neg\neg$ Elim
\end{fitch}$
\bigskip

\noindent Standard rules for lattice connectives:

$\begin{fitch}
\fh A\land B  \\
\fa A & $\land$ Elim
\end{fitch}$
\hspace{.75cm}
$\begin{fitch}
\fh A\land B  \\
\fa B & $\land$ Elim
\end{fitch}$
\hspace{.75cm}
$\begin{fitch}
\fa A  \\
\fj B  \\
\fa A\land B & $\land$ Int
\end{fitch}$

$\begin{fitch}
\fh  A\lor B  \\
\fa \fh A \\
\fa \fa \ldots \\
\fa \fa C \\
\fa \fh B \\
\fa \fa \ldots \\
\fa \fa C \\
\fa C & $\lor$ Elim
\end{fitch}$
\hspace{.75cm}
$\begin{fitch}
\fh A  \\
\fa A\lor B & $\lor$ Int 
\end{fitch}$
\hspace{.75cm}
$\begin{fitch}
\fh  B  \\
\fa A\lor B & $\lor$ Int
\end{fitch}$
\medskip

\noindent Negative rules for the lattice connectives\footnote{In the context of the positive rules, these negative rules are equivalent to Fitch's (\citeyear{fitch52}) original negation IntElim rules.}:

$\begin{fitch}
\fh \neg(A\lor B)  \\
\fa \neg A & $\neg\lor$ Elim
\end{fitch}$
\hspace{1cm}
$\begin{fitch}
\fh \neg(A\lor B)  \\
\fa \neg B & $\neg\lor$ Elim
\end{fitch}$

\bigskip
$\begin{fitch}
\fa \neg A  \\
\fj \neg B  \\
\fa \neg(A\lor B) & $\neg\lor$ Int
\end{fitch}$

$\begin{fitch}
\fh \neg A  \\
\fa \neg(A\land B) & $\neg\land$ Int 
\end{fitch}$
\hspace{2cm}
$\begin{fitch}
\fh \neg B  \\
\fa \neg(A\land B) & $\neg\land$ Int 
\end{fitch}$

\bigskip
$\begin{fitch}
\fh  \neg(A\land B)  \\
\fa \fh \neg A \\
\fa \fa \ldots \\
\fa \fa C \\
\fa \fh \neg B \\
\fa \fa \ldots \\
\fa \fa C \\
\fa C & $\neg\land$ Elim
\end{fitch}$

\bigskip
\noindent Rules for Conditional\footnote{The Dilemma rule, if added to, say, Intuitionistic Logic, would yield a formulation of full classical logic.  It does not collapse FDE+cmi into classical logic because the usual $\neg$-Introduction rule (\emph{reductio}) of Intuitionistic Logic is absent.}

\medskip
$\begin{fitch}
\fa (A\rightarrow B)  \\
\fj A \\
\fa B & $\rightarrow$ Elim
\end{fitch}$
\hspace{2cm}
$\begin{fitch}
\fa \fh A \\
\fa \fa \ldots \\
\fa \fa B \\
\fa A\rightarrow B & $\rightarrow$ Int
\end{fitch}$

\medskip
$\begin{fitch}
\fh \neg(A\rightarrow B)  \\
\fa A & $\neg\rightarrow$ Elim
\end{fitch}$
\hspace{1cm}
$\begin{fitch}
\fh \neg(A\rightarrow B)  \\
\fa \neg B & $\neg\rightarrow$ Elim
\end{fitch}$

\bigskip
$\begin{fitch}
\fa \fh A \\
\fa \fa \ldots \\
\fa \fa B \\
\fa \fh A\rightarrow C \\
\fa \fa \ldots \\
\fa \fa B \\
\fa B & Dilemma
\end{fitch}$
\hspace{2cm}
$\begin{fitch}
\fa  A \\
\fj \neg B  \\
\fa \neg(A\rightarrow B) & $\neg\rightarrow$ Int
\end{fitch}$

\newpage
\section*{Appendix II:  Soundness and Completeness}
The natural deduction system is provably sound and complete, where soundness is taken to mean that if all the premisses of a derivation have designated values on an assignment, the conclusion will as well.

Soundness can be verified in the usual way, arguing by induction on the size of the derivation after establishing that each rule is sound.  For rules in which a conclusion is inferred directly from one or two premisses, this is immediate, by inspection of the truth tables.  A rule in which a hypothesis is discharged (that is, in the terminology of Fitch-style natural deduction, a rule involving one or two \emph{subordinate proofs}) is considered sound just in case, if all the undischarged hypotheses above the conclusion of the rule (in Fitch-style: the formulas \emph{reiterated} into subordinate proofs) have designated values, the conclusion also has a designated value.  It is easy to see that the subproof rules of the system will be sound \emph{provided that} the reasoning within the subordinate proofs is sound.  The full soundness proof, then, will take the form of a double induction, on the length, and on the depth of nesting of subordinate proofs within, a proof.  The overall strategy is perfectly standard for soundness proofs of natural deduction systems.

Completeness can be proven by a variant of Henkin's method, similar to that used in \citet{Priest2018}.  We desire to show that, if a formula A is not derivable from a set of premisses $\Gamma$, then there is an assignment on which A takes an undesignated value but every member of $\Gamma$ is designated.  In a Henkin-style proof this is done in two stages.  In the first, it is shown that $\Gamma$ can be extended to an eligible set $\Gamma^*$ which still does not (syntactically) imply A, where a set of formulas is said to be eligible if it has some of the formal characteristics of the set of formulas taking designated values on some assignment.  In the second it is shown that the eligible set is actually elected: an assignment is defined on which all and only the members of the set take designated values.

In applying Henkin's method to a classical system, eligible sets are simply \emph{complete theories}, a.k.a. \emph{maximal consistent sets} of formulas.  In logics tolerating contradictions, consistency is obviously not a requirement, and in logics tolerating truth value ``gaps'' maximality is also not to be hoped for!  An appropriate notion of eligibility for our purposes counts a set of formulas as eligible if and only if (i) it is deductively closed (and so, in particular, contains a conjunction if and only if it contains both conjuncts, and contains a disjunction if it contains either disjunct), and (ii) contains at least one of the disjuncts of each disjunction it contains.  Given A not derivable from $\Gamma$, it is readily seen that a set maximal with respect to the property of containing $\Gamma$ but not implying A will be eligible in this sense, and the existence of such a maximal set follows from standard set-theoretic considerations (Teichm\"uller-Tukey lemma).  (In general these maximal sets will not be the only eligible supersets of $\Gamma$ not implying A, and an alternative proof adding formulas to $\Gamma$ only if they are required by clause (ii) may yield a smaller eligible set.)  

   Given an eligible set $\Gamma^*$, we define an assignment to propositional variables by setting 
\begin{itemize}
\item $v(p)=\Ta$ iff $p$ is a member but $\neg p$ is not a member of $\Gamma^*$,
\item $v(p)=\Ba$ iff $p$ and $\neg p$ both belong to $\Gamma^*$,
\item $v(p)=\Na$ iff neither $p$ nor $\neg p$ belongs to $\Gamma^*$, and
\item $v(p)=\Fa$ iff $\neg p$ but not $p$ is a member of $\Gamma^*$.
\end{itemize}
(Note that a variable takes a designated value if and only if it belongs to $\Gamma^*$.)  It remains to verify (by induction on formula complexity) that arbitrary formulas take values on this assignment under the same conditions of their, and their negations', membership in $\Gamma^*$.  None of the cases are hard; those not immediately obvious usually become obvious when one remembers that formulas of the form $A\lor(A\rightarrow B)$ are provable in the system.

\bigskip 
In a bit more detail:

\noindent {\bf Soundness:}

The notion of validity we want is:
\begin{definition}{{\bf Validity:}} If all the assumptions\footnote{Where formulas reiterated into a subproof are counted among that subproof's assumptions.} have designated values, then every formula or subordinate derivation  of the derivation has a designated value.
\end{definition}

One could easily check that all the rules are classically valid, and conclude that \emph{of course} the FDE+cmi system is sound.  But perhaps some of the \emph{negative} rules are worth checking.

\begin{table}[!h]
\begin{center}
\begin{tabular}{c || c c c c}
$\land$ & \T & \B & \N & \F \\
\hline\hline
\T & \T & \B & \N & \F \\
\B & \B & \B & \F & \F \\
\N & \N & \F & \N & \F \\
\F & \F & \F & \F & \F 
\end{tabular}
\caption{$\land$ truth-table}
\end{center}
\end{table}
\noindent A quick check of the $\neg\land$-Int rules against this truth table makes it clear that these rules are sound.  The presence of B means that $\neg\land$-Elim takes a bit longer, but is clearly correct also.

In full formal detail, the soundness proof (as is usual for a soundness proof of Fitch-style natural deduction systems) is a double induction on 
\begin{enumerate}
\item Length:  the number of non-assumption formula items in the derivation.
\item Depth: the depth of nesting of subproofs.
\end{enumerate}
The induction step of the Depth induction goes:
\begin{quote}
Suppose A comes by a subproof-using rule.  (Now, a number of cases.)  By hypothesis of induction on Length, anything reiterated into the subproof is ok (since A would occur in a main subproof above this step).  By hypothesis of induction on Depth [we call this the ``key step''], the subproof is sound.  So, if the hypothesis of the subproof is ok, so is its ``active'' last item.  So -- now verifying each case -- A is ok.

This gives us the cases for subproof-involving rules in the indicative step of the Length induction for derivations of depth $\le n$.  So derivations at this depth are sound.  In a derivation of depth $n+1$, every subproof is of depth $\le n$, so the ``key step'' is guaranteed.
\end{quote}

\bigskip
\noindent{\bf Completeness Construction:}

Define a set of formulas to be \emph{saturated} if and only if
\begin{enumerate}
\item  It is consistent \label{require0} in the sense of not being the set of all formulas, and
\item It is \emph{deductively closed}, and so automatically contains
\begin{enumerate}
\item A conjunction if and only if it contains both conjuncts, and
\item A disjunction if it contains at least one disjunct
\end{enumerate}
\item It contains at least one disjunct of each of its disjunctions
\end{enumerate}

Note that for \emph{classical} logic, in which by deductive closure $(A \lor \neg A)$ belongs to every saturated set, saturation amounts to \emph{maximal consistency}: saturated sets are the generalization for non-classical logics of the ``maximal consistent sets'' familiar from Henkin proofs for classic(-ally based) logic(s).

What we have to prove is: 
\begin{lemma}
For any set of formulas, $\Gamma$, and for any formula, $A$, not derivable from $\Gamma$, there is a \emph{saturated} superset of $\Gamma$, $\Gamma^*$, not containing $A$.  (Since this specification includes the requirement that $A\notin \Gamma^*$, the requirement of consistency, (\ref{require0}), doesn't \emph{have} to be made part of the \emph{definition} of saturated set.
\end{lemma}
\begin{proof}
(By a version of the standard Lindenbaum construction.)

For classical logic, where we want \emph{maximal} consistent sets anyway, it is normal to consider \emph{all} the formulas of the language (in some order), tossing each one in if its addition doesn't permit the derivation of {\sc the bad thing}.  But this can tend to \emph{undesirably stuffed} sets of formulas!  For example, the \emph{maximal} consistent sets of intuitionistic propositional logic are \emph{classically} maximal!  So we prefer a more cautious addition of formulas.  

\begin{definition}
An ordinally indexed series of sets of formulas:
\begin{quote}
Let $\Gamma_0 = \Gamma$

Assume some fixed well-ordering of the formulas of the language:

For \emph{odd successors}, $\alpha$,
\begin{itemize}
\item If $\Gamma_\alpha$ is saturated, stop
\item if not, pick the \emph{first} (in the assumed ordering) disjunction in $\Gamma_\alpha$ for which neither disjunct is in $\Gamma_\alpha$.  By the $\lor$-Elim rule, if $\Gamma_\alpha$ does not imply $A$, at least one of these disjuncts can be added to $\Gamma_\alpha$ without permitting the derivation of A.  Let $\Gamma_{\alpha+1}$ be the result of adding the first such disjunct (or the only one, if only one can be added) to $\Gamma_\alpha$.
\end{itemize}

For \emph{even successors} (and for 0), let $\Gamma_{\alpha+1}$ be the deductive closure of $\Gamma_\alpha$.

For \emph{limit ordinals} $\lambda$, let $\Gamma_\lambda = \bigcup\limits_{\alpha<\lambda}(\Gamma_\alpha)$.
\end{quote}

\end{definition}
By cardinality considerations (there are more ordinals than formulas, so eventually we will run out of formulas to add), this sequence must reach a fixed point, $\Gamma_\Omega$.
By the usual Henkin arguments, $\Gamma_\Omega$ is saturated and does not imply $A$.

Now, for \emph{atomic} $p$, define
\begin{itemize}
\item $p$ has the value \Ta iff $p$ does and $\neg p$ does not belong to $\Gamma_\Omega$
\item $p$ has the value \Ba iff both $p$ and $\neg p$ belong to $\Gamma_\Omega$
\item $p$ has the value \Na iff neither $p$ nor $\neg p$ belong to $\Gamma_\Omega$
\item $p$ has value \Fa iff $\neg p$ but not $p$ belongs to $\Gamma_\Omega$
\end{itemize}
We now in a position to \emph{verify}  that for arbitrary formulas $A$, the same correlation holds between value and status with respect to membership in $\Gamma_\Omega$

\end{proof}

\subsection*{Completeness Verification:}\label{sec:Verify}
\begin{theorem}
Given a \emph{saturated} set S, if we define, for atomic $p$:

$p\rightsquigarrow$\Ta iff $p\in S$ and $\neg p\notin S$

$p\rightsquigarrow$\Ba iff $p\in S$ and $\neg p\in S$

$p\rightsquigarrow$\Na iff $p\notin S$ and $\neg p\notin S$

$p\rightsquigarrow$\Fa iff $p\notin S$ and $\neg p\in S$

\noindent we will have the same coincidence of values and S-status for all formulas
\end{theorem}
\begin{proof}
By induction on formula structures.  We omit proofs for the $\neg,\lor,\land$ connectives, which are obvious. So we consider $\rightarrow$:

\noindent{\bf The way up}  Suppose equivalence holds for $A$ and $B$, we show it holds for $A\rightarrow B$.

\noindent There are 16 combinations of truth-like values for $A, B$.  For each, by hypothesis of induction, assume the right S-membership status.
\begin{itemize}
\item {\bf Cases}: Left column.  $B$ is \T, so is in S.  So, $A\rightarrow B$ is in S.  $\neg B$ is \emph{not} in S, so by simple $\neg\rightarrow$-Elim, $\neg(A\rightarrow B)$ can't be either.
\item {\bf Cases}: Bottom two rows.  A is either \Fa or \N, so by hypothesis of induction, A is \emph{not} in S.  By Dilemma, $A\lor(A\rightarrow B)$ is in S, so by saturation $(A\rightarrow B)\in S$.  In the other direction, simple $\neg\rightarrow$-Elim rule would get from $\neg(A\rightarrow B)$ to $A$, so $\neg(A\rightarrow B)$ can't be in S.
\item {\bf Cases}: Top row:  $A$ has value \T, so $A\in S, \neg A\notin S$
\begin{itemize}
\item {\bf subcase}: $B$ has value \B, so both $B\in S$ and $\neg B\in S$.  Since $B\in S$, $\rightarrow$-Int gives $(A\rightarrow B)\in S$.  Since $A\in S$ and $\neg B\in S$, the simple $\neg\rightarrow$-Int rule gives $\neg(A\rightarrow B)\in S$.
\item {\bf subcase}: $B$ has value \N, so neither $B\in S$ nor $\neg B\in S$.  If $(A\rightarrow B)$ were in S, $\rightarrow$-Elim would put $B$ in, so $(A\rightarrow B)\notin S$.  If $\neg(A\rightarrow B)$ were in S, simple $\neg\rightarrow$-Elim rule would put both $A$ and $\neg B$ in S, so $\neg(A\rightarrow B)\notin S$.
\item {\bf subcase}: $B$ has value \F, so $\neg B\in S$ and $B\notin S$.  If $(A\rightarrow B)$ were in S, $\rightarrow$-Elim would put $B$ in, so $(A\rightarrow B)\notin S$.  On the other hand, since $A$ and $\neg B$ are both available, simple $\neg\rightarrow$-Int gives $\neg(A\rightarrow B)\in S$.
\end{itemize}
\item {\bf Cases}: Second row: $A$ has value \B, so both $A\in S$ and $\neg A\in S$
\begin{itemize}
\item {\bf subcase}:  $B$ has value \B, so by hypothesis of induction, $B$ and $\neg B$ are both in S.  Since $B$ is available, $\rightarrow$-Int gets $(A\rightarrow B)\in S$.  Since $A$ and $\neg B$ are both available, simple $\neg\rightarrow$-Int puts $\neg(A\rightarrow B)\in S$.
\item {\bf subcase}: $B$ has value \N, so neither $B$ nor $\neg B$ is in S.  Since $A$ is available, if $(A\rightarrow B)$ were in S, $\rightarrow$-Elim would put $B$ in S.  If $\neg(A\rightarrow B)$ were in S, simple $\neg\rightarrow$-Elim would put $\neg B$ in S.  So neither $(A\rightarrow B)$ nor $\neg(A\rightarrow B$ is in S.
\item {\bf subcase}: $B$ has value \F, so $\neg B\in S$ and $B\notin S$.  Since $A$ and $\neg B$ are available, $\neg\rightarrow$-Int puts $\neg(A\rightarrow B)\in S$.  Since $A$ is available, if $A\rightarrow B$ were in S, $B$ would be also.
\end{itemize}
\end{itemize}
\end{proof}


Normally, in a Henkin-style completeness proof, one defines truth for atoms as membership in the saturated set, and the \emph{verification} stage checks that complex formulas are true if and only if they are members.  This requires \emph{two} arguments: one that \emph{if} a formula is true it is a member, and one that \emph{if} a formula belongs to the set, then it is true.

Because the semantic clauses for \T, \B, \N, \Fa are more complex, these two parts are now intermingled.  If a formula has a certain value, then both the membership and non-membership of the formula and its negation (one for each) in the set have to be checked.  Thus {\bf the way up} (verifying that formulas with certain values have the right membership statuses) already in effect includes {\bf the way back down} (verifying that formulas with certain membership statuses have the right values).

\end{document}